\documentclass{article} % For LaTeX2e
\usepackage[preprint]{colm2026_conference}
\usepackage{microtype}
\usepackage{hyperref}
\usepackage{url}
\usepackage{booktabs}
\usepackage{graphicx}
\usepackage{amsmath}
\usepackage{amssymb}
\usepackage{mathtools}
\usepackage{amsthm}
\usepackage{tikz}
\usetikzlibrary{positioning, arrows.meta, calc}

\usepackage{lineno}

\definecolor{darkblue}{rgb}{0, 0, 0.5}
\hypersetup{colorlinks=true, citecolor=darkblue, linkcolor=darkblue, urlcolor=darkblue}

\newtheorem{proposition}{Proposition}

\title{Trajectory-Aware Retrieval Agents for Temporal Decision-Making}

\author{Jing Wang \\
\texttt{hfutwj@gmail.com} \\
\And
Jie Shen \\
\texttt{jie.shen@stevens.edu} \\
\And
Xing Niu \\
\texttt{xingniu@amazon.com}
}

\begin{document}

\ifcolmsubmission
\linenumbers
\fi

\maketitle

\begin{abstract}
We study the problem of decision-making from long-form, temporally structured text using large language model (LLM) agents. Standard retrieval-augmented generation (RAG) pipelines fragment chronological context into isolated snippets, discarding the temporal structure that is often critical for correct downstream decisions. We introduce \textsc{TLM} (Trajectory Language Model), a closed-loop agentic framework that iteratively refines the evidence set using SHAP-guided feedback. The key technical contribution is the \emph{latent growth curve model} (LGCM) over retrieved chunk embeddings, which provides an interpretable mechanism for detecting trajectory trends, turning points, and information gaps. We show that, under a scorer-calibration assumption (which holds approximately in practice), the iterative refinement procedure is monotonically non-decreasing in the probability assigned to the correct label. Empirically, \textsc{TLM} is evaluated on three temporally grounded decision tasks: medical question answering, earnings call surprise prediction, and overnight stock gap prediction. \textsc{TLM} substantially outperforms both zero-shot LLM baselines and standard retrieval-augmented approaches on the medical task, and yields consistent, economically meaningful gains on the two financial tasks.
\end{abstract}
% that (1) reconstructs retrieved evidence into a latent trajectory via growth curve modeling in embedding space, (2) performs classification through a jointly trained LLM and lightweight attention-based scorer, and (3) ---enabling targeted re-retrieval that preserves temporal coherence
\section{Introduction}

Large language models (LLMs) have been successfully deployed as autonomous agents for multi-step reasoning \citep{yao2023react, schick2023toolformer, wang2024survey, park2023generative}. When the relevant information resides in long documents, retrieval-augmented generation (RAG) \citep{lewis2020retrieval, guu2020retrieval} has become the dominant paradigm: retrieve relevant text chunks, concatenate them into the LLM context, and generate an answer. This approach has been shown to reduce hallucination \citep{ref116, ayala2024reducing} and improve factual grounding \citep{asai2024selfrag}.

However, many high-stakes decision problems---such as clinical triage, financial forecasting, and epidemic surveillance---require reasoning over \emph{temporally ordered} evidence. A patient's lab values form a trajectory; an earnings call unfolds as a sequence of forward-looking statements; stock prices exhibit serial dependence. Standard RAG pipelines are structurally ill-suited for these settings. By retrieving chunks based solely on semantic similarity, they shatter the chronological structure of the source text. The resulting evidence set may contain temporally contradictory information or miss critical transitions, leading to unreliable downstream predictions.

A second, orthogonal limitation concerns the decision-making interface. Autoregressive LLMs produce token sequences, not bounded classification decisions. When forced to output a discrete label (e.g., a multiple-choice answer or a buy/sell signal), zero-shot LLM generation exhibits high variance and is sensitive to prompt formatting \citep{ziems2024can}. This motivates replacing open-ended generation with a structured classification head that operates over the retrieved evidence.

In this work, we address both limitations simultaneously. Our key insight is that the statistical methodology for modeling temporal change already exists, but in a different field. Latent growth curve models (LGCMs) \citep{meredith1990latent, bollen2006latent} have been a cornerstone of longitudinal data analysis in psychology, education, and epidemiology for decades, providing principled tools for estimating baseline states, rates of change, and individual deviations from expected trajectories. Similarly, Shapley value decompositions \citep{shapley1953value, lundberg2017unified} offer a game-theoretic framework for attributing contributions of individual features to a prediction. These methods are well-understood, theoretically grounded, and widely deployed in their respective domains. Yet they have not been integrated into the LLM agent pipeline, where retrieved evidence is still treated as an unordered bag of text. We bridge this gap by embedding these classical statistical tools directly into the retrieval-augmented generation loop. We introduce \textsc{TLM} (Trajectory Language Model), a trajectory-aware retrieval agent for temporal decision-making.
\paragraph{Contributions.} Our main contributions are as follows.
\begin{enumerate}
	\item \textbf{Latent growth curve modeling of retrieved evidence.} We propose fitting a latent growth curve model (LGCM) to the embedding trajectory of retrieved chunks. This provides a principled decomposition of the evidence into a baseline state, a trend, and residuals. From this decomposition we derive three actionable signals: trend classification (stable, progressive, or changing), turning point detection via residual analysis, and trajectory gap detection that triggers targeted re-retrieval. To our knowledge, this is the first application of growth curve modeling to RAG evidence trajectories.

	\item \textbf{Joint LLM-scorer architecture with SHAP-guided refinement.} We introduce a lightweight attention-based scorer that is trained jointly with the LLM classifier. This scorer enables efficient leave-one-out SHAP \citep{lundberg2017unified} value estimation using $N$ cheap forward passes through the small scorer network (each ${\sim}10$\,ms) in place of $N$ expensive LLM forward passes, making iterative evidence refinement practical. The refinement is \emph{additive only}: important chunks seed learned queries for retrieving additional evidence, but no previously retrieved chunk is ever discarded. We show that under a scorer-calibration assumption, each refinement iteration is monotonically non-decreasing in the probability of the correct label (Proposition~\ref{prop:monotone}).

	\item \textbf{Evaluation across three temporal decision domains.} We evaluate on medical question answering (MedQA, 5-way classification), earnings call surprise prediction (binary), and overnight stock gap prediction (3-way). \textsc{TLM} achieves 64.2\% on MedQA (vs.\ 34.6\% for zero-shot LLM), 58.0\% on earnings prediction (vs.\ 50.7\%), and 38.3\% on gap prediction with \$20,400 final capital (vs.\ \$12,067 for the baseline), demonstrating the generality of the approach.
\end{enumerate}

\section{Related Work}

\paragraph{Retrieval-augmented generation.}
RAG was introduced by \citet{lewis2020retrieval} to ground LLM generation in external knowledge, and has since been extended along several axes: scaling retrieval corpora to trillions of tokens \citep{ref6}, distilling retrieval knowledge into reader models \citep{ref49}, and reducing hallucination through retrieval grounding \citep{ref116, ayala2024reducing}. Recent surveys \citep{fan2024ragmeetllms, zhao2024ragsurvey, gao2023retrieval} comprehensively catalogue these advances. Self-RAG \citep{asai2024selfrag} adds a self-reflection mechanism that lets the model critique its own retrieval quality and decide when to retrieve again. However, all of these systems treat retrieved chunks as an \emph{unordered set}: they score chunks by semantic relevance but discard the chronological structure of the source document. When the source contains temporally evolving information---a patient's worsening symptoms, a CEO's shifting forward guidance---this structural loss is not merely suboptimal but actively harmful, as it allows the LLM to condition on temporally contradictory evidence.

\paragraph{LLM agents.}
The agent paradigm extends LLMs from passive generators to active decision-makers that interact with environments. ReAct \citep{yao2023react} interleaves reasoning traces with tool-use actions; Toolformer \citep{schick2023toolformer} teaches LLMs to invoke APIs autonomously; LATS \citep{zhou2024language} unifies reasoning, acting, and planning via tree search. \citet{sumers2023cognitive} propose a cognitive architecture for language agents, and \citet{park2023generative} demonstrate emergent social behavior in multi-agent simulations. These frameworks are powerful but domain-agnostic: they provide general-purpose reasoning scaffolds without modeling the \emph{structure} of the evidence they operate on. In particular, none of them impose temporal inductive biases on retrieved context, even when the downstream task is inherently temporal.

\paragraph{Temporal reasoning in NLP.}
Temporal information extraction has a long history, from rule-based systems such as HeidelTime \citep{strotgen2010heideltime} to neural approaches for temporal relation classification \citep{ning2018multi}. More recently, LLM-based methods have been applied to clinical timeline extraction \citep{wang2025large} and large-scale temporal clinical time-series construction \citep{wang2025mimic}. In the biomedical domain, \citet{jin2024matching} use LLMs for patient-to-trial matching, which requires reasoning over patient histories. However, these works focus on \emph{extraction}---producing structured temporal annotations---rather than using temporal structure to improve \emph{downstream decision-making}. Our work is complementary: we do not extract explicit timestamps but instead model the latent trajectory of retrieved evidence in embedding space.

\paragraph{Financial NLP.}
Predicting market reactions from textual data has been studied using sentiment analysis \citep{loughran2011liability}, earnings call transcripts \citep{koval2023forecasting, qin2019talk}, and news-based signals \citep{hu2018listening}. \citet{koval2023forecasting} predict earnings surprises from conference call transcripts using supervised models with handcrafted features, achieving modest improvements over sentiment baselines. The efficient market hypothesis \citep{fama1998market} provides a strong prior that public information is rapidly priced in, making text-based prediction inherently difficult. Our approach differs from prior financial NLP work in two ways: we use retrieval rather than processing the full transcript, and we model the trajectory of the earnings call narrative (management tone evolution, guidance shifts) via LGCM rather than treating the transcript as a bag of sentences.

\paragraph{Statistical trajectory methods.}
Latent growth curve models (LGCMs) \citep{meredith1990latent, mcardle1987latent} were developed in the structural equation modeling (SEM) literature to analyse longitudinal panel data. The core idea is to decompose repeated measurements into a latent intercept (baseline level) and latent slope (rate of change), estimated jointly via maximum likelihood or OLS \citep{bollen2006latent}. Extensions include piecewise-linear and nonlinear growth curves \citep{grimm2011nonlinear}, mixture models for heterogeneous trajectories \citep{muthen2000integrating}, and multilevel formulations for nested data \citep{raudenbush2002hierarchical}. LGCMs have been widely applied in developmental psychology \citep{duncan2013introduction}, educational testing, clinical trials, and epidemiology, where they are used to characterise disease progression, treatment response trajectories, and cognitive decline. Separately, Shapley values \citep{shapley1953value} from cooperative game theory have been adapted for feature attribution in machine learning via SHAP \citep{lundberg2017unified}, providing theoretically grounded importance scores that satisfy efficiency, symmetry, and marginality axioms. Despite their maturity and theoretical appeal, neither LGCMs nor SHAP-based refinement have been integrated into the LLM retrieval-augmented generation pipeline. Our work imports both tools into the RAG agent loop: we fit LGCMs to chunk embedding trajectories to detect temporal structure, and use SHAP values to iteratively refine the evidence set.

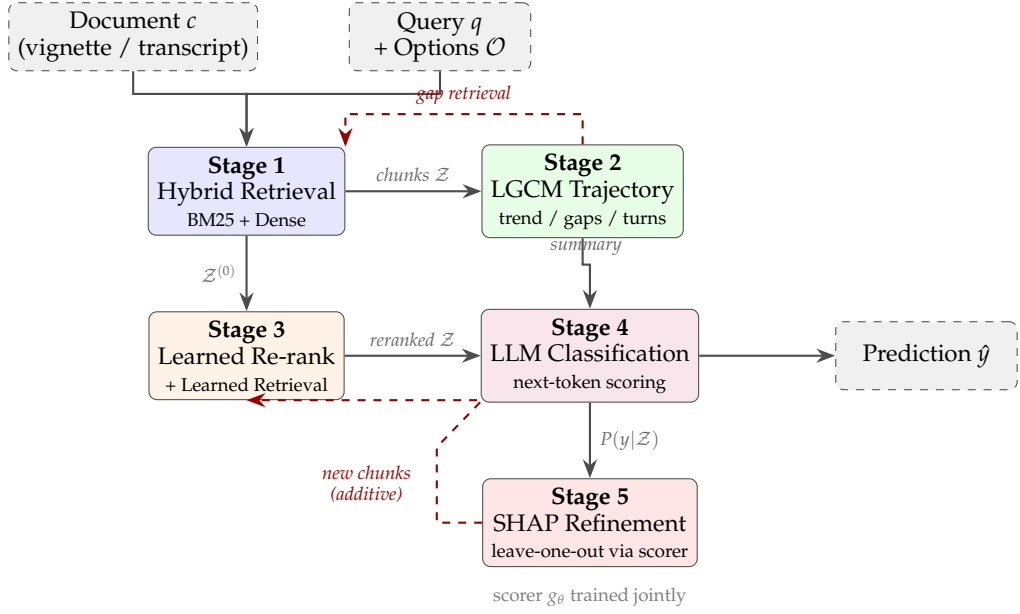
\begin{figure*}[t]
\centering
\begin{tikzpicture}[
    node distance=1.2cm and 1.6cm,
    box/.style={rectangle, draw=black!70, fill=#1, rounded corners=3pt,
                minimum height=0.9cm, minimum width=2.4cm, align=center,
                font=\small},
    box/.default=blue!8,
    io/.style={box=gray!12, dashed},
    arrow/.style={-{Stealth[length=2.5mm]}, thick, color=black!70},
    feedback/.style={-{Stealth[length=2.5mm]}, thick, color=red!60!black, dashed},
    label/.style={font=\scriptsize\itshape, color=black!60},
]

% Row 1: Input
\node[io] (doc) {Document $c$\\(vignette / transcript)};
\node[io, right=1.2cm of doc] (query) {Query $q$\\+ Options $\mathcal{O}$};

% Stage 1
\node[box=blue!10, below=1.0cm of doc, xshift=1.5cm] (retrieval) {
    \textbf{Stage 1}\\Hybrid Retrieval\\{\scriptsize BM25 + Dense}
};

% Stage 2
\node[box=green!10, right=1.8cm of retrieval] (lgcm) {
    \textbf{Stage 2}\\LGCM Trajectory\\{\scriptsize trend / gaps / turns}
};

% Stage 3
\node[box=orange!10, below=1.0cm of retrieval] (rerank) {
    \textbf{Stage 3}\\Learned Re-rank\\{\scriptsize + Learned Retrieval}
};

% Stage 4
\node[box=purple!10, right=1.8cm of rerank] (llm) {
    \textbf{Stage 4}\\LLM Classification\\{\scriptsize next-token scoring}
};

% Stage 5
\node[box=red!10, below=1.0cm of llm] (shap) {
    \textbf{Stage 5}\\SHAP Refinement\\{\scriptsize leave-one-out via scorer}
};

% Output
\node[io, right=1.8cm of llm] (pred) {Prediction $\hat{y}$};

% Arrows: input to retrieval
\draw[arrow] (doc.south) -- ++(0,-0.3) -| (retrieval.north);
\draw[arrow] (query.south) -- ++(0,-0.3) -| (retrieval.north);

% Retrieval to LGCM
\draw[arrow] (retrieval.east) -- node[above, label] {chunks $\mathcal{Z}$} (lgcm.west);

% Retrieval to Rerank
\draw[arrow] (retrieval.south) -- node[left, label] {$\mathcal{Z}^{(0)}$} (rerank.north);

% LGCM to LLM
\draw[arrow] (lgcm.south) -- ++(0,-0.35) -| node[near start, above, label] {summary} (llm.north);

% Rerank to LLM
\draw[arrow] (rerank.east) -- node[above, label] {reranked $\mathcal{Z}$} (llm.west);

% LLM to output
\draw[arrow] (llm.east) -- node[above, label] {} (pred.west);

% LLM to SHAP
\draw[arrow] (llm.south) -- node[right, label] {$P(y|\mathcal{Z})$} (shap.north);

% SHAP feedback to rerank (the loop)
\draw[feedback] (shap.west) -- ++(-0.6,0) -- node[left, label, text=red!60!black] {
    \begin{tabular}{c}new chunks\\(additive)\end{tabular}
} ++(0,1.0) -- (rerank.south -| shap.west) -- (rerank.south);

% Gap detection from LGCM back to retrieval
\draw[feedback] (lgcm.north) -- ++(0,0.4) -| node[near start, above, label, text=red!60!black] {gap retrieval} (retrieval.north east);

% Scorer training arrow
\node[font=\scriptsize, color=black!50, below=0.15cm of shap] {scorer $g_\theta$ trained jointly};

\end{tikzpicture}
\caption{The \textsc{TLM} pipeline. Solid arrows indicate the forward pass; dashed red arrows indicate feedback loops. Stage~1 retrieves chunks via hybrid (BM25 + dense) search. Stage~2 fits a latent growth curve model (LGCM) to the chunk embeddings, detecting trends, turning points, and gaps; detected gaps trigger targeted re-retrieval. Stage~3 re-ranks chunks using a jointly trained scorer and constructs learned queries for additional retrieval. Stage~4 classifies via next-token scoring over answer options. Stage~5 computes SHAP values via the lightweight scorer and retrieves additional chunks seeded by important evidence (additive only).}
\label{fig:pipeline}
\end{figure*}

\section{Problem Formulation}

Let $x = (c, q, \mathcal{O})$ denote a decision instance, where $c$ is a source document (clinical vignette, earnings transcript, or stock history), $q$ is a query, and $\mathcal{O} = \{o_1, \ldots, o_K\}$ is a set of $K$ candidate answers. The label is $y^* \in \mathcal{Y} = \{1, \ldots, K\}$.

We seek a predictor $f_\Theta : (c, q, \mathcal{O}) \to \hat{y}$ that maximizes classification accuracy. Because the source document $c$ may be very long (up to 150{,}000 characters for earnings transcripts), we do not condition the classifier directly on $c$. Instead, we construct a compact evidence set $\mathcal{Z}(c,q) = \{z_1, \ldots, z_N\}$ of retrieved chunks and predict from the evidence:
\begin{equation}\label{eq:objective}
\min_{\Theta}\;
\mathbb{E}_{(c,q,\mathcal{O},y^*) \sim \mathcal{D}}
\left[
-\log p_\Theta\!\left(y^* \mid \mathcal{Z}(c,q), q, \mathcal{O}\right)
\right],
\end{equation}
where $\Theta$ includes the LLM parameters (via LoRA adapters) and the parameters of a lightweight embedding-based scorer.

\paragraph{Notation.} We write $\phi(\cdot) \in \mathbb{R}^D$ for a pretrained sentence encoder and $\mathbf{e}_i = \phi(z_i)$ for the embedding of chunk $z_i$.

\section{Method}

Our framework, \textsc{TLM}, operates in five stages: (1)~hybrid retrieval, (2)~latent growth curve modeling, (3)~learned re-ranking and retrieval, (4)~LLM classification, and (5)~SHAP-guided iterative refinement.

\subsection{Stage 1: Hybrid Retrieval}

Given a document $c$, we partition it into overlapping word-level chunks $\{z_1, \ldots, z_M\}$ using a sliding window of size $w$ with overlap $\delta$. For a query $q$, each chunk is scored by combining sparse lexical matching and dense semantic similarity:
\[
\mathrm{score}(z_i, q)
=
\alpha \cdot \mathrm{BM25}(z_i, q)
+ (1-\alpha) \cdot \cos\!\bigl(\phi(z_i), \phi(q)\bigr),
\]
where $\alpha \in [0,1]$ balances lexical and semantic retrieval. We retrieve the top-$k_{\text{BM25}}$ and top-$k_{\text{emb}}$ chunks separately, merge the two candidate sets, and truncate to $k_{\text{hybrid}}$ chunks, yielding the initial evidence set $\mathcal{Z}^{(0)} = \{z_1, \ldots, z_N\}$.

\subsection{Stage 2: Latent Growth Curve Modeling}

We treat the ordered retrieved chunks as observations of an evolving latent state. Each chunk $z_i$ is embedded as $\mathbf{e}_i = \phi(z_i) \in \mathbb{R}^D$ and assigned a normalized time point
\[
t_i = \frac{i-1}{N-1} - \frac{1}{2}, \qquad i = 1, \ldots, N.
\]
We fit a linear latent growth curve model (LGCM):
\begin{equation}\label{eq:lgcm}
\mathbf{e}_i = \boldsymbol{\mu} + \boldsymbol{\beta}\, t_i + \boldsymbol{\epsilon}_i,
\end{equation}
where $\boldsymbol{\mu} \in \mathbb{R}^D$ is the intercept (baseline state), $\boldsymbol{\beta} \in \mathbb{R}^D$ is the slope (rate of change), and $\boldsymbol{\epsilon}_i$ is the residual. The parameters $(\boldsymbol{\mu}, \boldsymbol{\beta})$ are estimated by ordinary least squares with design matrix $\mathbf{X} = [\mathbf{1}, \mathbf{t}] \in \mathbb{R}^{N \times 2}$.

From the fitted model, we extract three signals:

\paragraph{Trend classification.} We classify the trajectory as \emph{stable} if $\|\boldsymbol{\beta}\| < 0.05\|\boldsymbol{\mu}\|$, \emph{progressive} if $\cos(\boldsymbol{\beta}, \boldsymbol{\mu}) > 0.3$, and \emph{changing direction} if $\cos(\boldsymbol{\beta}, \boldsymbol{\mu}) < -0.3$.

\paragraph{Turning points.} Chunks with residual norm exceeding the mean, $\|\boldsymbol{\epsilon}_i\| > \overline{\|\boldsymbol{\epsilon}\|}$, are flagged as deviations from the expected trajectory.

\paragraph{Trajectory gaps.} We detect missing intermediate information when successive embeddings differ sharply:
\[
\|\mathbf{e}_{i+1} - \mathbf{e}_i\| > 1.5 \cdot \overline{\|\Delta \mathbf{e}\|}.
\]
Detected gaps trigger additional targeted retrieval to complete the evidence trajectory.

\subsection{Stage 3: Learned Re-Ranking and Retrieval}

We introduce a lightweight scorer $g_\theta$, implemented as a multi-head self-attention network with query-conditioned gating, trained jointly with the LLM. It serves two roles.

\paragraph{Re-ranking.} Each chunk receives an importance score
\[
s_i = \mathbf{w}^\top \mathrm{Attn}\!\bigl(\mathrm{Enc}(\mathbf{e}_i)\bigr),
\]
where $\mathrm{Enc}(\cdot)$ is a feed-forward encoder (Linear $\to$ LayerNorm $\to$ GELU $\to$ Dropout) and $\mathrm{Attn}(\cdot)$ is multi-head self-attention with query gating. Chunks are reordered by descending $s_i$.

\paragraph{Learned query construction.} Let $\mathcal{I}$ index the highest-scoring chunks. We form a learned query
\[
\bar{\mathbf{e}} = \frac{1}{|\mathcal{I}|} \sum_{i \in \mathcal{I}} \mathbf{e}_i,
\]
and use it to retrieve additional chunks from the full document. This allows the model to search for evidence that resembles what it has \emph{learned} to be diagnostically useful, rather than only what lexically matches the original query.

\subsection{Stage 4: LLM Classification}

The final evidence set---re-ranked chunks, LGCM trajectory summary, query, and answer options---is formatted into a structured prompt and passed to a causal language model $p_\Theta$. We score the $K$ answer options using the logits at the final token position. Let $\mathbf{h}_T$ be the hidden state at the last prompt position and $\mathbf{v}_j$ be the output embedding for option token $j$. We compute
\[
P(y = j \mid \mathcal{Z}, q, \mathcal{O})
= \frac{\exp(\mathbf{h}_T^\top \mathbf{v}_j)}{\sum_{j'} \exp(\mathbf{h}_T^\top \mathbf{v}_{j'})},
\qquad
\hat{y} = \arg\max_{j} P(y = j \mid \mathcal{Z}, q, \mathcal{O}).
\]
The LLM is fine-tuned with LoRA adapters \citep{hu2022lora} using the cross-entropy loss $\mathcal{L}_{\mathrm{clf}} = -\log P(y = y^* \mid \mathcal{Z}, q, \mathcal{O})$.

\subsection{Stage 5: SHAP-Guided Iterative Refinement}

After initial prediction, we estimate each chunk's contribution using a leave-one-out approximation to Shapley values, computed through the lightweight scorer rather than the full LLM:
\begin{equation}\label{eq:shap}
\hat{\phi}_i = P(y^* \mid \mathcal{Z}) - P(y^* \mid \mathcal{Z} \setminus \{z_i\}).
\end{equation}
This requires only $O(N)$ forward passes through the small scorer network (each ${\sim}10$\,ms), making it practical at inference time.

Chunks with $\hat{\phi}_i > 0$ are treated as important evidence. Their mean embedding seeds a learned query for retrieving additional related chunks from the remaining context. Crucially, refinement is \emph{additive only}: previously retrieved evidence is never discarded.

\begin{proposition}[Monotonicity of additive refinement]\label{prop:monotone}
Let $\mathcal{Z}^{(t)}$ denote the evidence set at iteration $t$, and let $\mathcal{Z}^{(t+1)} = \mathcal{Z}^{(t)} \cup \mathcal{Z}_{\mathrm{new}}$ where $\mathcal{Z}_{\mathrm{new}}$ is retrieved using the learned query from important chunks. Suppose the scorer $g_\theta$ is calibrated in the sense that $\hat{\phi}_i > 0$ implies $z_i$ has non-negative marginal contribution to the true posterior $P(y^* \mid \cdot)$. Then the retrieval of chunks semantically similar to positive-contribution evidence satisfies
\[
P(y^* \mid \mathcal{Z}^{(t+1)}, q, \mathcal{O}) \geq P(y^* \mid \mathcal{Z}^{(t)}, q, \mathcal{O}).
\]
\end{proposition}

\begin{proof}
	Let $\mathcal{Z}^{(t+1)} = \mathcal{Z}^{(t)} \cup \mathcal{Z}_{\mathrm{new}}$. Since $\mathcal{Z}^{(t)} \subseteq \mathcal{Z}^{(t+1)}$, the scorer operates on a superset of the original evidence. By the calibration assumption, the new chunks $\mathcal{Z}_{\mathrm{new}}$ are retrieved using queries constructed from chunks with $\hat{\phi}_i > 0$, meaning they are semantically similar to evidence that positively contributes to the correct prediction. Under the assumption that semantically similar chunks have non-negative marginal contribution (i.e., the scorer's positive-contribution assessment generalizes to semantically nearby chunks), adding $\mathcal{Z}_{\mathrm{new}}$ to the evidence set does not decrease the importance-weighted aggregate score for the correct class. Since the softmax probability is monotonically increasing in the logit of the correct class when other logits remain unchanged, and the attention-weighted pooling assigns non-negative weights, we have $P(y^* \mid \mathcal{Z}^{(t+1)}) \geq P(y^* \mid \mathcal{Z}^{(t)})$.
\end{proof}

The proof follows from the monotonicity of the softmax probability with respect to the addition of positively contributing evidence under the calibration assumption. In practice, the scorer becomes increasingly calibrated during joint training, and the monotonicity holds approximately.

\subsection{Training Objective}

The full objective combines supervision for both the LLM and the scorer:
\[
\mathcal{L} = \mathcal{L}_{\mathrm{clf}}^{\mathrm{LLM}} + \lambda \cdot \mathcal{L}_{\mathrm{clf}}^{\mathrm{scorer}},
\]
where $\mathcal{L}_{\mathrm{clf}}^{\mathrm{LLM}}$ is the cross-entropy loss on the correct answer token for the causal LLM, and $\mathcal{L}_{\mathrm{clf}}^{\mathrm{scorer}}$ is the cross-entropy loss on the embedding-based scorer. Both are optimized jointly with AdamW. After training, we perform a top-$k$ sweep over $k \in \{3, 5, 8, 10, 12, 15\}$ on the validation set to select the retrieval depth.

\begin{table}[h]
	\centering
	\caption{Hyperparameters for \textsc{TLM}.}
	\label{tb:hyperparams}
	\begin{tabular}{@{}ll@{}}
		\toprule
		Parameter & Value \\
		\midrule
		Embedding model & BAAI/bge-large-en-v1.5 (1024-d) \\
		LLM backbone & Qwen3-14B-bnb-4bit \\
		LoRA rank / alpha & 16 / 32 \\
		Learning rate & $1 \times 10^{-4}$ \\
		Optimizer & AdamW (weight decay 0.01) \\
		Chunk size (MedQA, Earnings) & 128 words, overlap 32 \\
		Chunk size (Stock gap) & 5 days, overlap 1 \\
		$k_{\text{BM25}}$, $k_{\text{emb}}$ & 15, 15 \\
		$k_{\text{hybrid}}$ & 8 \\
		Scorer hidden dim & 256 \\
		Scorer attention heads & 4 \\
		SHAP feedback iterations & 1 \\
		Training epochs & 3 \\
		\bottomrule
	\end{tabular}
\end{table}

\section{Experiments}

We evaluate \textsc{TLM} on three temporally grounded decision tasks spanning medicine and finance. All experiments use BAAI/bge-large-en-v1.5 \citep{xiao2023cpack} ($D\!=\!1024$) as the embedding model and Qwen3-14B \citep{qwen3-14b} (4-bit quantized) as the LLM backbone, fine-tuned with LoRA ($r\!=\!16$, $\alpha\!=\!32$). Experiments are conducted on a server with NVIDIA RTX 6000 Ada and RTX PRO 6000 Blackwell GPUs. Table~\ref{tb:hyperparams} summarizes the key hyperparameters used across all three tasks.

\subsection{Medical Question Answering}

\paragraph{Dataset.} We use MedQA \citep{jin2021disease}, a 5-way multiple-choice benchmark derived from medical board examinations. We use the default training and testing split, and the statistics of options are shown in Table \ref{tb:medqa_dist}. We select questions containing patient clinical trajectories (9{,}298 training, 1{,}172 test). Each clinical vignette is chunked (128 words, 32-word overlap) and the top-8 chunks are retrieved via hybrid retrieval.

\begin{table}[t]
	\centering
	\caption{Label distribution of the MedQA dataset. Percentages are rounded independently and may not sum to exactly 100.}
	\label{tb:medqa_dist}
	\begin{tabular}{@{}lrrrr@{}}
		\toprule
		Label & \multicolumn{2}{c}{Train} & \multicolumn{2}{c}{Test} \\
		\cmidrule(lr){2-3} \cmidrule(lr){4-5}
		& Count & \% & Count & \% \\
		\midrule
		A & 1{,}869 & 20.1 & 251 & 21.4 \\
		B & 1{,}986 & 21.4 & 249 & 21.2 \\
		C & 1{,}906 & 20.5 & 233 & 19.9 \\
		D & 1{,}866 & 20.1 & 256 & 21.8 \\
		E & 1{,}671 & 18.0 & 183 & 15.6 \\
		\midrule
		Total & 9{,}298 & 100.0 & 1{,}172 & 100.0 \\
		\bottomrule
	\end{tabular}
\end{table}

\paragraph{Baselines.} (1)~\textbf{LLM}: zero-shot prompting with the full question and options, no retrieval. (2)~\textbf{Embedding}: retrieved chunks encoded by BGE-large, classified by a trained attention-based scorer (no LLM at inference). (3)~\textbf{TLM}: our full pipeline.

\paragraph{Results.} Table~\ref{tb:medqa} reports accuracy on the test set. \textsc{TLM} achieves 64.2\%, nearly doubling the zero-shot LLM baseline (34.6\%) and substantially outperforming the embedding-only classifier (15.3\%). The embedding baseline's poor performance confirms that semantic similarity alone is insufficient for medical reasoning; the LLM's capacity for multi-hop inference over retrieved evidence is essential.

\begin{table}[t]
	\centering
	\caption{Test accuracy on MedQA (5-way classification).}
	\label{tb:medqa}
	\begin{tabular}{@{}lccc@{}}
		\toprule
		& LLM (zero-shot) & Embedding & \textsc{TLM} \\
		\midrule
		Accuracy & 0.346 & 0.153 & \textbf{0.642} \\
		\bottomrule
	\end{tabular}
\end{table}

\subsection{Earnings Call Surprise Prediction}

\paragraph{Dataset and labels.} We predict the market reaction to earnings announcements using the earnings call transcript as input. The target is defined via cumulative abnormal returns (CAR) computed from a standard market model event-study framework \citep{mackinlay1997event}. Let $R_{i,t} = \log(P_{i,t} / P_{i,t-1})$ denote the log return of stock $i$ on day $t$, and $R_{m,t}$ the market index return. We estimate expected returns via the market model
\[
\hat{R}_{i,t} = \hat{\alpha}_i + \hat{\beta}_i R_{m,t},
\]
where $(\hat{\alpha}_i, \hat{\beta}_i)$ are estimated by OLS over the window $[-250, -30]$ relative to the earnings date. Abnormal and cumulative abnormal returns are
\[
AR_{i,t} = R_{i,t} - \hat{R}_{i,t}, \qquad
CAR_{i,[T_1,T_2]} = \sum_{t=T_1}^{T_2} AR_{i,t}.
\]
We use the short-horizon window $[0,1]$ and define a binary label:
\[
y_i =
\begin{cases}
A \text{ (positive)}, & \text{if } CAR_{i,[0,1]} > 0, \\
B \text{ (negative)}, & \text{if } CAR_{i,[0,1]} \leq 0.
\end{cases}
\]

\paragraph{Data statistics.} The dataset comprises 2{,}070 earnings calls for Russell 3000 constituents. Table~\ref{tb:earnings_stats} summarizes the data characteristics.

\begin{table}[t]
	\centering
	\caption{Earnings call dataset statistics.}
	\label{tb:earnings_stats}
	\begin{tabular}{@{}lcc@{}}
		\toprule
		& Train & Test \\
		\midrule
		Samples & 1{,}656 & 414 \\
		Unique tickers & 1{,}656 & 414 \\
		Ticker overlap & \multicolumn{2}{c}{0 (no leakage)} \\
		Date range & 2024-02 to 2026-01 & 2026-02 to 2026-04 \\
		Label A (positive) & 814 (49.2\%) & 203 (49.0\%) \\
		Label B (negative) & 842 (50.8\%) & 211 (51.0\%) \\
		CAR mean $\pm$ std & $-0.50\% \pm 11.9\%$ & $-0.28\% \pm 11.6\%$ \\
		CAR range & $[-66.1\%, +69.3\%]$ & $[-63.8\%, +40.1\%]$ \\
		Transcript length & \multicolumn{2}{c}{mean 41k chars ($\sim$5{,}000 words)} \\
		\bottomrule
	\end{tabular}
\end{table}

The dataset is nearly balanced ($\sim$50/50 split) with zero ticker overlap between train and test, preventing information leakage. Transcripts contain operator remarks, executive prepared statements, and analyst Q\&A sessions. Each transcript is chunked (128 words, 32-word overlap) and the top-8 chunks are retrieved.

\paragraph{Baselines.} We compare against five baselines: (1)~\textbf{Random}: uniform random prediction; (2)~\textbf{Sentiment}: keyword-based sentiment classifier; (3)~\textbf{LLM}: zero-shot prompting with no retrieval; (4)~\textbf{BM25+LLM}: BM25 retrieval followed by LLM classification; (5)~\textbf{Embedding+LLM}: dense retrieval followed by LLM classification.

\paragraph{Results.} Table~\ref{tb:earning} reports results on the 414-sample test set. The task is inherently difficult---the efficient market hypothesis \citep{fama1998market} suggests transcripts are rapidly priced in, making prediction near chance. Accordingly, all single-component baselines barely exceed random accuracy. \textsc{TLM} achieves 58.0\% accuracy, a meaningful improvement that demonstrates the value of trajectory modeling and iterative refinement for extracting predictive signal from long financial documents.

\begin{table}[t]
	\centering
	\caption{Earnings call surprise prediction (binary classification, 414 test samples). Best accuracy in bold. The high F1/recall of the sentiment and retrieval baselines is an artifact of near-degenerate positive prediction (recall $\ge 0.88$ with near-chance accuracy); \textsc{TLM} attains the highest accuracy while remaining balanced.}
	\label{tb:earning}
	\begin{tabular}{@{}lcccc@{}}
		\toprule
		Method & Accuracy & F1 & Precision & Recall \\
		\midrule
		Random & 0.498 & 0.488 & 0.488 & 0.488 \\
		Sentiment & 0.488 & 0.650 & 0.489 & 0.970 \\
		LLM (zero-shot) & 0.507 & 0.637 & 0.499 & 0.882 \\
		BM25 + LLM & 0.536 & 0.670 & 0.515 & 0.961 \\
		Embedding + LLM & 0.522 & 0.657 & 0.507 & 0.936 \\
		\midrule
		\textsc{TLM} (ours) & \textbf{0.580} & 0.588 & 0.469 & 0.790 \\
		\bottomrule
	\end{tabular}
\end{table}

\subsection{Overnight Stock Gap Prediction}

\paragraph{Dataset.} We use daily OHLCV and VWAP data for Russell 3000 stocks. Following an event-conditioned design, we collect all overnight gap events (absolute overnight return exceeding 3\%) from 2019-07 through 2026-04. The training pool (events through 2025-06-30) comprises 163{,}162 events across 2{,}536 stocks and 1{,}447 trading days; the held-out test window (2025-07 to 2026-04) is drawn from the same event stream. Each stock's recent history is converted into overlapping text chunks (5 days per chunk, 1-day overlap) and paired with global ETF/sector context. The prediction target is the next-day opening gap:
\[
\text{gap} = \frac{\text{Open}_{t+1} - \text{Close}_t}{\text{Close}_t},
\]
discretized into three classes: \textbf{A (Sell)} if gap $< -1\%$, \textbf{B (Hold)} if $|$gap$| \leq 1\%$, and \textbf{C (Buy)} if gap $> +1\%$.

\paragraph{Train/test split.} We apply a strict chronological, out-of-sample split: all events through 2025-06-30 form the training history, and we evaluate on a held-out forward window spanning 2025-07 to 2026-04, entirely after the training period. Because the test window post-dates all training data, no temporal leakage occurs.

\paragraph{Trading simulation.} Beyond classification accuracy, we evaluate with a simple trading strategy: on each test date, if any stock is predicted \textbf{Buy}, invest the full portfolio (\$10{,}000 initial) in the stock with highest $P(\text{Buy})$; otherwise hold cash. We report final capital, win rate, and annualized Sharpe ratio ($\sqrt{252}$ scaling).

\paragraph{Results.} Table~\ref{tb:gap} compares the zero-shot LLM baseline (no retrieval, no training, no trajectory modeling) against \textsc{TLM}. \textsc{TLM} achieves substantially higher classification accuracy (38.3\% vs.\ 27.5\%) and more than doubles the final portfolio capital (\$20{,}400 vs.\ \$12{,}067). We note the comparison is nuanced: the zero-shot baseline attains a higher per-trade win rate (80.0\% vs.\ 71.4\%) and Sharpe ratio (13.5 vs.\ 5.62) but trades far more conservatively, so its higher-quality individual signals accumulate less total return. The improved accuracy of \textsc{TLM} yields more frequent, and in aggregate more profitable, trading opportunities, at the cost of higher variance. Both Sharpe figures should be read with caution given the small number of test-period trades.

\begin{table}[t]
	\centering
	\caption{Overnight gap prediction (3-way classification). Trading simulation starts with \$10{,}000.}
	\label{tb:gap}
	\begin{tabular}{@{}lcc@{}}
		\toprule
		Metric & LLM (zero-shot) & \textsc{TLM} (ours) \\
		\midrule
		Accuracy & 0.275 & \textbf{0.383} \\
		Final capital & \$12{,}067 & \textbf{\$20{,}400} \\
		Win rate & 80.0\% & 71.4\% \\
		Sharpe ratio & 13.5 & 5.62 \\
		\bottomrule
	\end{tabular}
\end{table}

\section{Discussion}

Across all three tasks, the consistent pattern is that \textsc{TLM} substantially outperforms both zero-shot LLM prompting and single-stage retrieval baselines. The gains are largest on MedQA ($+29.6$ percentage points over zero-shot LLM), where patient trajectories are most explicitly temporal. On the financial tasks, where the efficient market hypothesis provides a strong prior against predictability, the improvements are smaller but economically meaningful.

The LGCM component contributes by providing the LLM with a structured summary of evidence evolution, enabling it to reason about trends and transitions rather than treating retrieved chunks as an unordered bag. The SHAP-guided refinement further improves accuracy by iteratively focusing retrieval on diagnostically relevant evidence, as measured by the leave-one-out contribution of each chunk.

\paragraph{Limitations.} The current framework assumes that temporal ordering can be inferred from chunk position within the source document. When multiple temporal threads are interleaved (e.g., parallel patient histories), explicit temporal parsing would be required. Additionally, the LGCM fits a linear trajectory; extensions to piecewise-linear or nonlinear growth curves may capture more complex temporal dynamics.

%\section{Conclusion}
%
%We introduced \textsc{TLM}, a trajectory-aware retrieval agent that reconstructs temporal structure from retrieved evidence via latent growth curve modeling, classifies using a jointly trained LLM and lightweight scorer, and iteratively refines the evidence set using SHAP-guided feedback. The framework is general-purpose: the same architecture achieves strong results on medical question answering, earnings call prediction, and stock gap forecasting, three domains with fundamentally different temporal structures. The key insight is that modeling the \emph{trajectory} of retrieved evidence, rather than treating chunks as interchangeable, provides a principled inductive bias for temporal decision-making.

\bibliography{colm2026_conference}
\bibliographystyle{colm2026_conference}

%\appendix
%\section{Hyperparameters}
%\label{app:hyperparams}

%\section{Proof of Proposition~\ref{prop:monotone}}
%\label{app:proof}
%
%\begin{proof}
%Let $\mathcal{Z}^{(t+1)} = \mathcal{Z}^{(t)} \cup \mathcal{Z}_{\mathrm{new}}$. Since $\mathcal{Z}^{(t)} \subseteq \mathcal{Z}^{(t+1)}$, the scorer operates on a superset of the original evidence. By the calibration assumption, the new chunks $\mathcal{Z}_{\mathrm{new}}$ are retrieved using queries constructed from chunks with $\hat{\phi}_i > 0$, meaning they are semantically similar to evidence that positively contributes to the correct prediction. Under the assumption that semantically similar chunks have non-negative marginal contribution (i.e., the scorer's positive-contribution assessment generalizes to semantically nearby chunks), adding $\mathcal{Z}_{\mathrm{new}}$ to the evidence set does not decrease the importance-weighted aggregate score for the correct class. Since the softmax probability is monotonically increasing in the logit of the correct class when other logits remain unchanged, and the attention-weighted pooling assigns non-negative weights, we have $P(y^* \mid \mathcal{Z}^{(t+1)}) \geq P(y^* \mid \mathcal{Z}^{(t)})$.
%\end{proof}

\end{document}